\documentclass[lettersize,journal]{IEEEtran}

\usepackage{amsmath, amssymb, amsthm}
\usepackage{graphicx,tikz}
\usetikzlibrary{positioning, calc}
\usepackage{booktabs} 
\usepackage{multirow, multicol}

\usepackage{thmtools}
\usepackage{thm-restate} 
\declaretheorem[name=Theorem, numberwithin=section]{theorem}
\declaretheorem[name=Proposition, sibling=theorem]{proposition}

\theoremstyle{plain}

\theoremstyle{definition}

\newtheorem{assumption}[theorem]{Assumption}
\theoremstyle{remark}

\renewcommand{\vector}[1]{{\bf #1}}

\newcommand{\DA}{\Delta}
\newcommand{\trace}[1]{\mbox{tr}(#1)}

\newcommand{\bmad}{b_{\operatorname{MAD}}}
\newcommand{\bmed}{b_{\operatorname{med}}}
\newcommand{\sigmamad}{\sigma_{\operatorname{MAD}}}
\newcommand{\sigmamed}{\sigma_{\operatorname{med}}}

\newcommand{\R}{\mathbb{R}}
\newcommand{\I}{\mathcal{I}}
\newcommand{\Out}{\mathcal{O}}

\begin{document}

\title{SNAP: A Self-Consistent Agreement Principle \\ with Application to Robust Computation}

\author{Xiaoyi Jiang, Andreas Nienkötter}

\author{Xiaoyi Jiang, Andreas Nienkötter
        \thanks{Xiaoyi Jiang is with Faculty of Mathematics and Computer Science, University of Münster, Münster, Germany. Andreas Nienkötter is with Institute for Disaster Management and Reconstruction, Sichuan
University-Hongkong Polytechnic University, Chengdu, Sichuan, China.}
    }



\maketitle

\begin{abstract}
We introduce \textbf{SNAP} (Self-coNsistent Agreement Principle), a self-supervised framework for robust computation based on mutual agreement. Based on an Agreement–Reliability Hypothesis SNAP assigns weights that quantify agreement, emphasizing trustworthy items and downweighting outliers without supervision or prior knowledge. A key result is the \emph{Exponential Suppression of Outlier Weights}, ensuring that outliers contribute negligibly to computations, even in high-dimensional settings. We study properties of SNAP weighting scheme and show its practical benefits on vector averaging and subspace estimation. Particularly, we demonstrate that non-iterative SNAP outperforms the iterative Weiszfeld algorithm and two variants of multivariate median of means. SNAP thus provides a flexible, easy-to-use, broadly applicable approach to robust computation.

\end{abstract}

\section{Introduction}

In many computational and machine learning tasks, multiple candidate solutions, model predictions, or observed entities are available. Some are reliable, while others are noisy or erroneous. We observe a general pattern: \emph{reliable entities tend to agree with each other, whereas unreliable ones are dispersed}. We refer to this as the Agreement–Reliability Hypothesis (ARH) and argue that leveraging this structural property can significantly enhance robust computation. While the agreement–dispersion assumption is canonical in robust statistics, it is also evident in computational results. For example, ensemble methods in classification \cite{Xue2026} and clustering \cite{Zhang2022} postprocess computed outputs. In this work, we adopt a broad view encompassing both data and computation results. Although aspects of ARH have been noted previously, e.g. \cite{Franek2010}, it has not been systematically studied, nor has a general, self-supervised, task-agnostic framework been developed to exploit it.

We introduce the \textbf{Self-coNsistent Agreement Principle (SNAP)}, a general, self-supervised framework for robust computation based on mutual agreement. It consists of an agreement-based weighting mechanism and a weight-guided computation scheme. SNAP assigns weights by quantifying how consistently entities agree with one another, automatically emphasizing trustworthy items and downweighting unreliable ones. This provides robustness through collective consistency, without requiring supervision or prior knowledge.

\begin{figure}[t]
\centering
\makebox{
\includegraphics[width=0.5\linewidth]{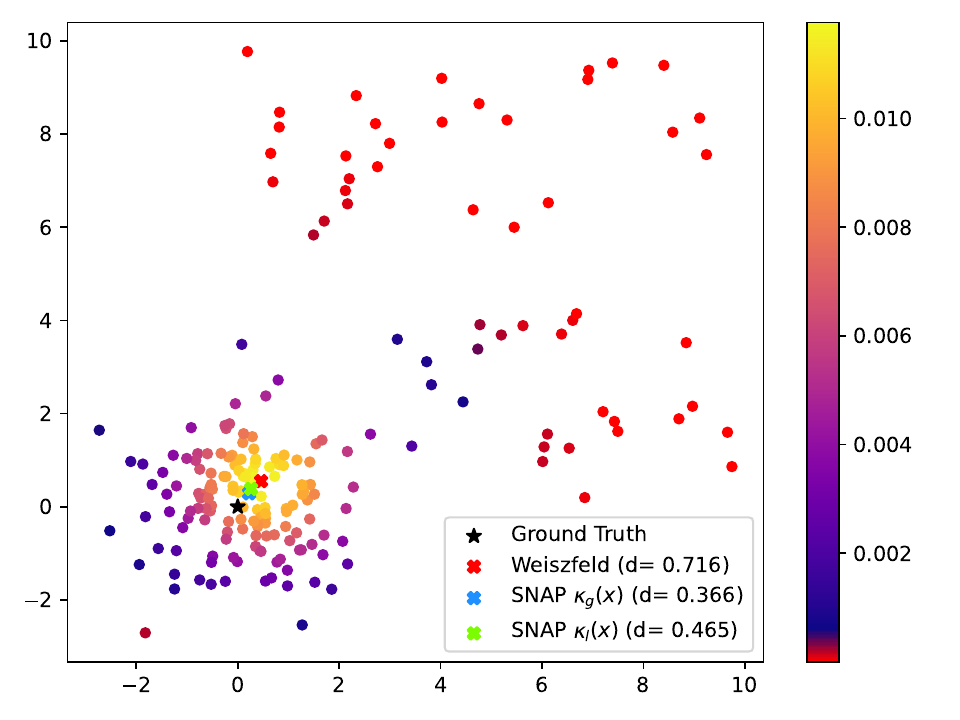}
\includegraphics[width=0.5\linewidth]{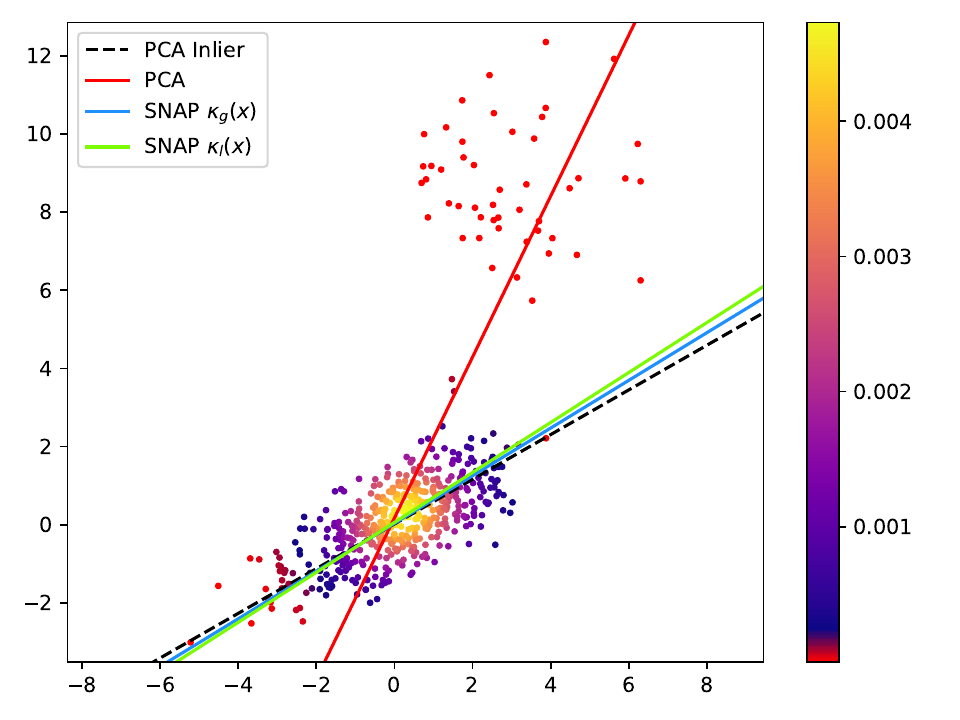}
}
\caption{Illustration of SNAP. Inliers have high mutual agreement and contribute strongly to the weighted consensus, while outliers contribute less. Data points are colored by their SNAP weight with $\kappa_l(\vector{x})$. Note the outliers have weights near zero, highlighted in red. Left: 2D vector averaging. Weiszfeld: Geometric median of all points; SNAP $\kappa_g$, SNAP $\kappa_l$: Robust averaging using all points with SNAP weights. Distance to ground truth is in brackets. Right: PCA.}
\label{fig:illustration}
\end{figure}

This paper focuses on the weighting mechanism of the framework. We propose a novel agreement-based weighting method and systematically study its properties. A key result is the \emph{Exponential Suppression of Outlier Weights} (Theorem~\ref{thm:key_result}), which shows that the influence of outliers vanishes at an exponential or super-exponential rate, making their contribution to the final computation negligible. Figure~\ref{fig:illustration} illustrates the core ideas of SNAP in these settings: outliers are highly dispersed and receive negligible weight, so the computation is driven primarily by inliers, yielding robust solutions even in high-dimensional regimes.



SNAP is a unifying principle for agreement-based robust computation with provable properties. Focusing on the weighting mechanism, this paper makes the following contributions:
1) \textbf{SNAP} concept and the agreement-based weighting scheme, which is data-driven, fully parameter-free;
2) Theoretical results that give a deeper understanding of the SNAP weighting scheme;
3) Brief presentation of SNAP weight-based robust computation (overview of the full framework) and validation of SNAP weighting on two problems: aggregation and subspace estimation;
4) Demonstration for vector averaging in $\mathbb{R}^m$ that non-iterative SNAP outperforms the widely used iterative Weiszfeld algorithm \cite{beck_weiszfelds_2015} and two variants of multivariate median of means. We envision establishing a general framework for SNAP-based easy-to-use robust computation. This paper presents its theoretical foundations along with initial applications.

\section{Related Work}

\noindent
{\bf Importance score computation.}  
Measuring the importance of data is key for computation. Centrality in network analysis \cite{Wan2021} is a prominent example, and importance scores are widely used in machine learning, e.g. graph attention networks \cite{Velickovic2018}. Distance-based importance measures relate closely to kernel density estimation \cite{Bishop2006}. SNAP weighting is motivated differently: it aims to distinguish inliers from outliers. While its form resembles softmax-normalized importance scores, there is a fundamental difference, which we discuss in detail in Section~\ref{sec:comparison}.

\noindent
{\bf Robust statistics.}  
Classical robust estimators modify the loss function, e.g. minimizing $\sum_i \phi(L_i)$ with robust penalties $\phi$ (Huber, Tukey) to reduce the influence of large residuals \cite{Maronna2019}. All samples contribute, but with bounded influence. SNAP, in contrast, achieves robustness via adaptive sample reweighting, optimizing $\sum_i w_i L_i$ with $w_i \in [0,1]$ determined from the data, downweighting outliers automatically.

\noindent
{\bf Weighted computation.}  
Many weighted computation schemes rely on domain-specific rules to handle input heterogeneity, e.g. in federated learning \cite{Gadekallu2026}. SNAP provides a general weighting framework to distinguish inliers from outliers, independent of task or domain. Although weighted computation is widely used \cite{Zhang2022}, the weight computation itself has rarely been systematically studied.

Further discussion of related work on robust estimation, geometry, and agreement principles in a broader context is provided in the Appendix \ref{app:related_work}.

\section{Agreement–Reliability Hypothesis}
\label{sec:arh}

SNAP is motivated by the Agreement–Reliability Hypothesis (ARH) based on
a simple and widely observed phenomenon: when multiple entities attempt to describe or estimate the same underlying quantity, those that are reliable tend to be similar, whereas unreliable ones tend to disagree with each other and also with the reliable entities. In such settings, agreement is a natural indicator of reliability, while dispersion signals uncertainty or noise. SNAP does not assume that all entities are reliable or that agreement is perfect; rather, it leverages the overall pattern of agreement and disagreement to infer consensus in a robust manner.

We formally define the ARH, which forms the foundation of the SNAP framework.

\noindent
{\bf Assumption (Agreement–Reliability Hypothesis).}
Let $\mathcal{E} = \{o_i\}_{i=1}^n$ in space $(\mathcal{X}, d)$ denote a set of consensus entities. There exists an unknown subset of inliers $\I$ and outliers $\I^c = \Out$ such that:
1) (\emph{Inlier agreement}) For all $i,j \in \I$, $d(o_i, o_j)$ is small with high probability.
2) (\emph{Outlier dispersion}) For any $i \in \Out$ and for most $j \in \{1,\dots,n\}$, $d(o_i, o_j)$ is large.
In particular, inlier entities form a concentrated cluster in $\mathcal{X}$, whereas outlier entities are not mutually concentrated and do not form a competing cluster. This assumption is satisfied in many practical situations, as observed in a variety of aggregation, estimation, and learning tasks. We will further quantify this assumption in Section \ref{sec:key_result}.

\section{Self-Consistent Agreement Principle}
\label{sec:SNAP_Principle}

\subsection{Agreement Weights}
\label{sec:consensus_weights}

Given $\mathcal{E}$, SNAP assigns a non-negative weight to each entity that reflects its level of agreement with the rest of the collection. These \emph{agreement weights} quantify how well each entity fits the assumption that mutually consistent entities correspond to reliable information, while dispersed entities are likely to be outliers or unreliable.
Formally, SNAP defines a weight vector
$\vector{w} = (w_1,\dots,w_n) \in \R^n$
satisfying three properties: 
1) $w_i \ge 0$ for all $i$;
2) $\sum_{i=1}^n w_i = 1$;
3) Agreement-dependence: $w_i$ increases with the degree of agreement between $o_i$ and other entities in $\mathcal{E}$.
{\em SNAP weights are not probabilities, but influence coefficients} that control how much each object should influences (contribute to) the robust computation. As such, they can be directly used in weighted computation procedures.

To instantiate the SNAP principle, we measure agreement through pairwise distances in space $(\mathcal{X}, d)$. In general, we do not assume the distance function $d(x,y)$ to be metric, but only require $d(x,x)=0$. For each entity $o_i$, we define a normalized disagreement score:
\begin{equation}
\DA_i \ = \ \frac{D_i}{D_a = \sum_{i=1}^n D_i} \ = \ \frac{\sum_{j\not=i} d(o_i,o_j)}{\sum_{i=1}^n\sum_{j=1}^n d(o_i,o_j)}
\label{eqn:disagreement_score}
\end{equation}
The symbol $a$ in $D_a$ symbolizes all pairs. This disagreement score is then converted into a (normalized) agreement weight:
\begin{equation}
w_i = \frac{\kappa(\DA_i)}{\sum_{j=1}^n \kappa(\DA_j)}
\label{eqn:agreement_weight}
\end{equation}
where $\kappa(x)$ is a decreasing kernel function. Entities that are mutually consistent with many others obtain higher agreement weights.

\subsection{Formal Consideration of $\kappa$ Function}

We consider Laplacian and Gaussian:
\begin{equation}
e^{-x/\sigma}, \quad e^{-x^2/{\sigma^2}}
\end{equation}
$\sigma>0$ and use them to show theoretical results. Note that the constant factor before the exponential function is ignored since it disappears when computing the weights $w_i$.

For practical implementation we model the situation more precisely. The object with the smallest disagreement score $\DA_{\mbox{sm}}$, also called set median \cite{Nienkotter23}, is the most representative object (see Proposition \ref{prop:set_median}). Then, we model $\kappa$ with {\em half} Laplacian or Gaussian:
\[
\kappa_l(\DA) = e^{-\frac{\DA-\DA_{\mbox{sm}}}{b}}
, \:
\kappa_g(\DA) = e^{-\frac{\left(\DA-\DA_{\mbox{sm}}\right)^2}{2 \sigma^2}}
; \: \DA \geq \DA_{\mbox{sm}}
\]
Let $\DA_i^* = \DA_i - \DA_{\mbox{sm}}$. For robust computation we apply robust estimation for half distributions \cite{Maronna2019,gui2014generalization_slash_half_normal,alvarez2023modified}:
\[ \bmad \ = \ 1.4427 \ \mbox{MAD}, \quad \sigmamad \ = \ 2.2631 \ \mbox{MAD} \]
where MAD is Median Absolute Deviation: $\mbox{MAD} \ = \ \mbox{med}(\DA^* - \mbox{med}(\DA^*))$. There is another robust estimator:
\[ \bmed \ = \ 1.4427 \ \mbox{med}(\DA^*), \quad \sigmamed \ = \ 1.4826 \ \mbox{med}(\DA^*) \]
We will study both variants in our experiments. Given this formal consideration of $\kappa$ function, {\em the computation of agreement weights $w_i$ is completely parameter-free}.

\subsection{Comparison with Softmax-Normalized Importance Scores}
\label{sec:comparison}

Softmax-normalized importance scores over nodes are widely used in machine learning, e.g. in energy-based models \cite{LeCun2006} and graph attention networks \cite{Velickovic2018}, where exponentiated energies or distances induce probability distributions over states or nodes. A popular weighting scheme assigns weights to objects based on the sum of distances to all others:
\[
s_i = \sum_j d(x_i, x_j), \quad
w_i = \frac{k(s_i)}{\sum_j k(s_j)}
\]
Such schemes depend on the {\em absolute scale of distances}. As a result, weights is not scaling invariant.

The SNAP approach introduces a {\em fundamental novelty} by normalizing distances relative to the entire dataset, see \eqref{eqn:disagreement_score}, \eqref{eqn:agreement_weight}.
This simple modification yields {\em scale‑invariant, globally informed weights}. Weights reflect an object’s distance relative to the entire dataset, automatically downweighting outliers in a way that accounts for overall structure rather than absolute magnitude. Any decreasing kernel can be used to adjust robustness without manual outlier thresholds, unlike fixed trimming or thresholding methods \cite{Maronna2019}.
In summary, SNAP moves from absolute to relative distance weighting, providing robust, interpretable, and adaptable weights, which is a novel contribution beyond existing distance‑based schemes.

\subsection{Properties of Agreement Weights}

SNAP agreement weights provide a compressed, relational encoding of mutual agreement. We present several properties that offer further insight and will be useful for analysis. For completeness, proofs that are not trivial are provided in the appendix. None of the proofs assume any specific geometric structure beyond the non-negativity of the disagreement score, and most properties hold for arbitrary spaces and are highlighted in bold for clarity.

\subsubsection{Simple Properties}

\begin{proposition} (Geometric Invariance).
In $\R^m$ the agreement weights are invariant to translation, rotation, and scaling.
\end{proposition}

\begin{proposition} ({\bf Agreement Uniformity}).
If all objects are in perfect mutual agreement, meaning all pairwise disagreement measures are equal, then SNAP assigns uniform agreement weights:
$\forall i \not= j, \ d(o_i,o_j)=c \ \Longrightarrow \ w_i = \frac{1}{n}$.
\end{proposition}
This conceptually perfectly fits the intuition. Note that this agreement uniformity is necessary but not sufficient.

\begin{proposition}[{\bf Disagreement Monotonicity}]
If $D_i \le D_k$, then $w_i \ge w_k$.
\end{proposition}

\begin{proposition}[{\bf Maximal Agreement Weight}]
\label{prop:set_median}
The set median (object with the smallest disagreement score) has the largest agreement weight:
$\arg\min_i D_i \ = \ \arg\max_i w_i$.
\end{proposition}
Here for simplicity we assume there is a single set median. It is natural to assign this most representative object from the input set the largest agreement weight.

\begin{proposition}(Non-Monotonicity under Motion).
\label{prop:non-monotonicity}
Fix all points except $\vector{x}_k$ and define a continuous path:
$\vector{x}_k(t) = (1-t)\vector{x}_k + t \vector{x}_i, \quad t\in[0,1]$,
from $\vector{x}_k$ to some other $\vector{x}_i$. The agreement weight of $\mathbf{x}_k(t)$ changes non-monotonically on the path from $\vector{x}_k$ to $\vector{x}_i$.
\end{proposition}

We give a simple example by considering three 2D points: $\vector{x}_1 = (0,1)$, $\vector{x}_2 = (0,0)$, $\vector{x}_3 = (1,1)$ and the path from $\vector{x}_2$ to $\vector{x}_3$.
Figure \ref{fig:non-monotony} shows the agreement weights for $\vector{x}_1$, $\vector{x}_2(t)$, $\vector{x}_3$, where the squared Euclidean distance $\ell_2^2$: $d(\vector{x}, \vector{y})=||\vector{x}-\vector{y}||^2$ is used. When moving $\vector{x}_2$ toward $\vector{x}_3$ its agreement weight changes smoothly, but non-monotonically from $w(\vector{x}_2(0))\!\approx\!0.32$ to $w(\vector{x}_2(1))\!\approx\!0.36$. When $\vector{x}_2$ moves, the other two points change their agreement weights accordingly.

\begin{figure}[t]
    \centering
    \includegraphics[width=0.6\linewidth]{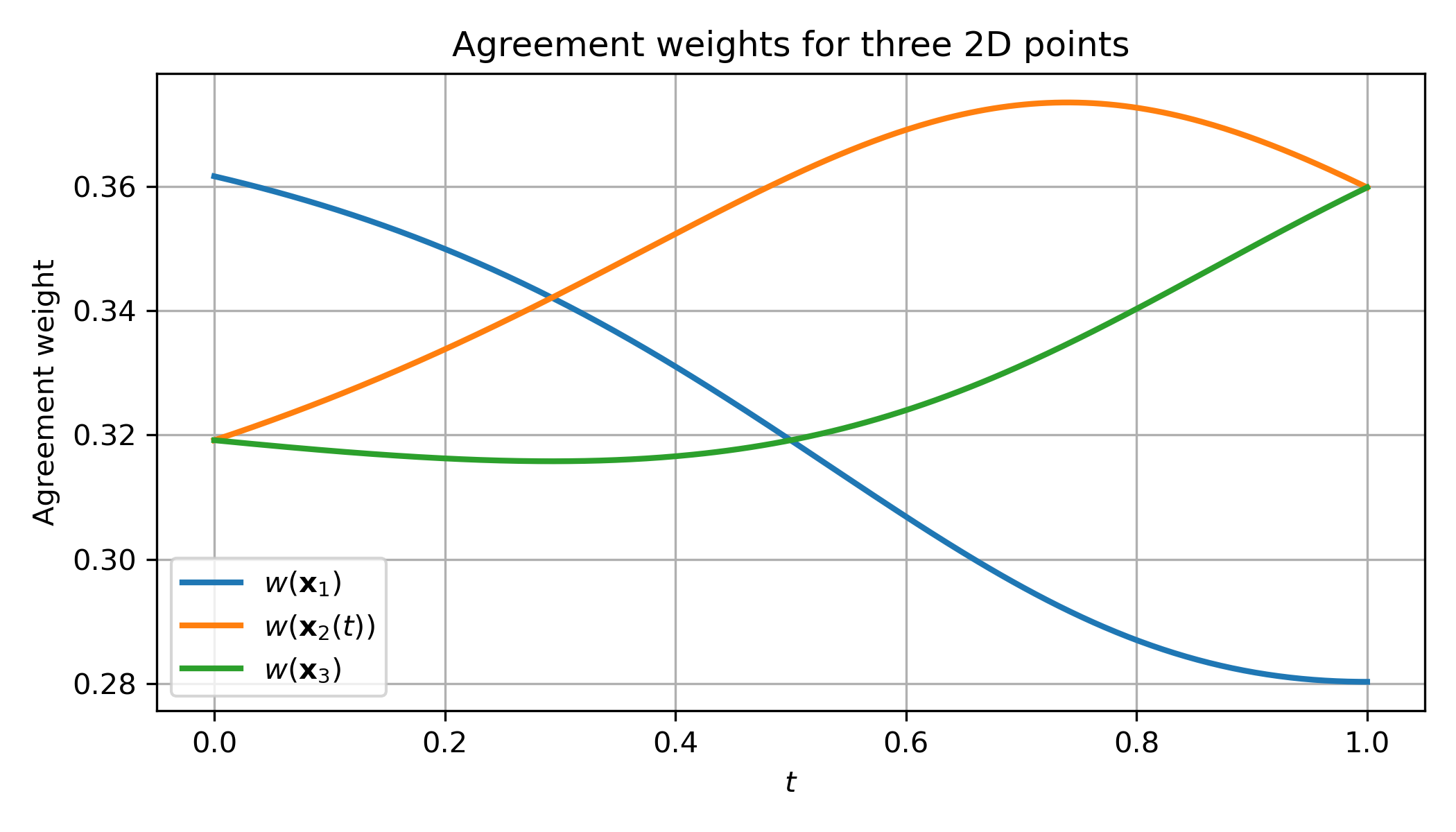}
    \vspace{-2mm}
    \caption{Non-monotonic evolution of agreement weights under continuous point motion.}
    \label{fig:non-monotony}
\end{figure}

While monotonicity w.r.t. distance may be intuitively expected, its absence is desirable here. The non-monotonicity stems from global coupling: moving a point closer to one point may move it farther from many others, and agreement weights depend on the entire configuration rather than pairwise distances alone. As a result, $w(\vector{x}_k(t))$ varies smoothly but not necessarily monotonically. This is desirable, since agreement weighting reflects global agreement rather than local attraction; a point may become less influential even as its distance decreases, capturing changes in the overall group structure and distinguishing our approach from local kernel or neighborhood-based methods.

\subsubsection{Local Sensitivity Analysis}

\begin{proposition}[restate=propSensitivity, name=Local Sensitivity Analysis]
\label{prop:sensitivity}
Let $K = \sum_{l=1}^n \kappa(\DA_l)$. Then $w_i = \frac{\kappa(\DA_i)}{K}$.
The Jacobian $J = [ \partial w_i / \partial \DA_j ]$ is given by:
\begin{equation} \label{eqn:jacobian}
    J_{ij} = \frac{\partial w_i}{\partial \DA_j} =
    \begin{cases}
        \dfrac{\kappa'(\DA_i)(K - \kappa(\DA_i))}{K^2}, & i=j \\[1em]
        -\dfrac{\kappa(\DA_i) \kappa'(\DA_j)}{K^2}, & i \neq j
    \end{cases}
\end{equation}
For any decreasing $\kappa(x)$, the diagonal elements are negative, while the off-diagonal elements are positive.
\end{proposition}

If $\kappa(\DA) = e^{-c \DA}$ ($c$ constant), then
$w_i = e^{-c\DA_i}/\sum_j e^{-c\DA_j}$.
In this case, the Jacobian becomes:
\begin{equation}
    \frac{\partial w_i}{\partial \DA_j} =
    \begin{cases}
        -w_i(1 - w_i), & i = j \\
        w_i w_j, & i \neq j
    \end{cases}
\end{equation}
See Appendix \ref{app:sensitivity} for proof.
The agreement weight redistribution is smooth and globally coupled. The Jabobian matrix encodes the mutual dependence of weights. The negative diagonal elements indicate that increasing one $\DA_i$ reduces $w_i$ (and slightly increases other weights). On the other hand, the positive off-diagonal elements mean that increasing $\DA_j$ of another point $\vector{x}_j$ increases $w_i$. If $\vector{x}_j$ becomes more isolated, $\vector{x}_i$’s relative weight increases. These facts confirm our intuition (expectation) on the agreement weights.

The result of sensibility analysis enables to prove additional properties. We show one such example.

\begin{proposition}[restate=propDisagreementSensitivity, name={\bf Strict Disagreement Sensitivity}]
\label{prop:disagreement_sensitivity}
Fix all objects except $o_k$, and let $o_k'$ be a perturbed version of $o_k$ such that
$d(o_k', o_j) > d(o_k, o_j) \ \text{for all } j\neq k
$.
Then, $w_k' < w_k$.
\end{proposition}
See Appendix \ref{app:disagreement_sensitivity} for proof.
It indicates that if a single object becomes strictly more distant from all other objects, then its agreement weight strictly decreases.

\subsubsection{Global Stability Analysis}

Due to its simplex structure the agreement weight is intrinsically bounded and has thus global Lipschitz stability property (see Appendix \ref{app:global_bound} for proof).
\begin{proposition}[restate=propGlobalBound, name={\bf Global Bounded Variation}]
\label{prop:global_bound}
The agreement weights $\vector{w} = (w_1, \dots, w_n)$ change to $\vector{w}^\prime$ when moving each point $\vector{x}_k$ to $\vector{x}_k^\prime, \ k=1,\dots,n$. Then
$||\vector{w}^\prime-\vector{w}|| \leq \sqrt{2}$ holds,
independent of the ambient dimension and the magnitude of the input perturbation.
\end{proposition}

\subsection{Exponential Suppression of Outlier Weights}
\label{sec:key_result}

\begin{proposition}[restate=propOutlierAmplification, name={\bf Suppression of Outlier Weights}]
\label{prop:outlier_amplification}
Fix all objects except $o_k$ and let
$d(o_k,o_{i \not= k}) \to\infty$. Then $w_k \to 0$ if $n\to\infty$.
\label{prop:outlier_amplification}
\end{proposition}
\vspace{-2mm}
See Appendix \ref{app:outlier_amplification} for proof. Now we further concretize the ARH. We assume that inliers form a tightly clustered majority and that outliers are sufficiently far from the inlier cluster such that the following holds.

\begin{assumption}[Linear Disagreement Gap]
\label{ass:linear_gap}
Let $\{o_i\}_{i=1}^n \subset (X,d)$ satisfy the ARH with inlier set $\I$, $|\I|=\alpha n$ for some $\alpha>1/2$. We assume that there exists a constant $c_0>0$, depending only on the structural parameters $(\alpha,r,R)$ (r and R defined below) such that for every inlier
$i\in \I$ and every outlier $k \in \Out$:
\[
\sum_{j=1}^n d(o_k,o_j)
\;-\;
\sum_{j=1}^n d(o_i,o_j)
\;\ge\;
c_0\, n
\]
\end{assumption}


The Linear Disagreement Gap captures the fact that outliers participate in a linear number of pairwise distances that are much larger than those among inliers. In contrast, inliers form a tightly clustered majority, so their cumulative disagreement is comparatively small. This imbalance naturally causes SNAP to assign exponentially smaller weights to outliers: each outlier “sticks out” in the population by contributing disproportionately to the total disagreement, while inliers reinforce each other. The exponential or Gaussian kernel then translates this linear gap into a strong suppression of outlier weights.

\begin{theorem}[restate=Suppression, name={\bf Exponential Suppression of Outlier Weights}]
\label{thm:key_result}
There exist constants $c>0$ and $n_0\in \mathbb{N}$
such that for all $n\ge n_0$ and all $k \in \Out$:
\[
w_k \le e^{-c n}, \quad w_k \le e^{-c n^2}
\]
for the kernel function $k(x)=e^{-x/\sigma}$ and $k(x)=e^{-x^2/\sigma^2}$, respectively.
\end{theorem}

See Appenddix \ref{app:key_result} for proof. The exponential bound on the weight $w_k$ assigned to any outlier formalizes SNAP's ability to distinguish inliers from outliers. As the sample size $n$ increases, the influence of outliers vanishes at an exponential or super-exponential rate, making their contribution to the final computation negligible. This underscores SNAP as a general, robust weighting framework rather than a domain-specific heuristic.

\subsection{Scalability of Agreement Weight Computation}
\label{sec:scalability}

\begin{figure}
    \centering
    \includegraphics[width=\linewidth]{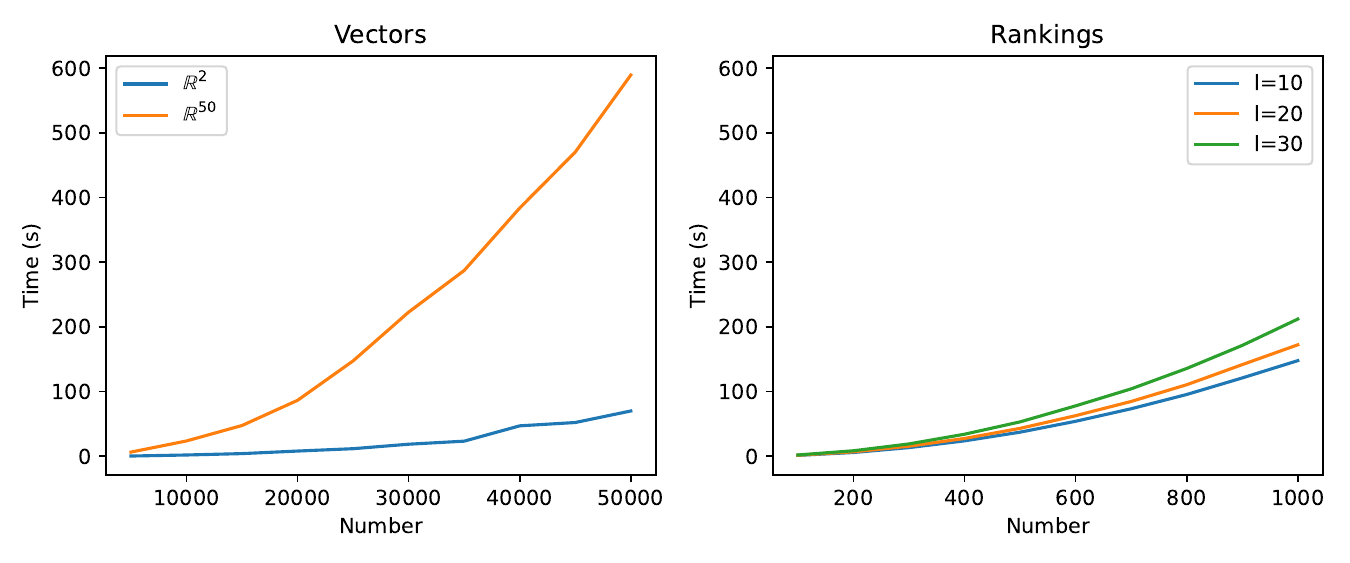}
    \caption{Run-time (in seconds) for agreement weight computation for vectors in $\R^m$ and rankings of length $l$.\label{fig:run-time}}
\end{figure}

Given $n$ entities in $(\mathcal{X},d)$ the computation of agreement weights needs $O(n^2 T)$ time, where $T$ is the overhead for one single distance computation. Figure~\ref{fig:run-time} shows the computation time of SNAP weights in the case of vectors in $\R^m$ using the Euclidean distance (in $\mathcal{O}(m)$ time) and rankings using the Kendall-Tau distance \cite{christensen2005kendalltau} (in $\mathcal{O}(l\log{l})$ time) where $l$ is the number of ranked items. All results were computed in Python 3.13 using a
standard pc with Intel i7-12700F CPU, 4.8GHz with 64GB RAM.
These results show the feasibility of agreement weight computation for a rather large number of objects. Approximate approaches like random subsampling could be used if speedup is required.

\section{SNAP Model for Robust Computation}

Given the set of consensus entities, a computation problem is defined to find the optimal parameters $\theta \in \Theta$. Let $L()$ be the loss function, we consider the Empirical Risk Minimization (ERM) problem:
$\min_{\theta \in \Theta} F(\theta) = \sum_{i=1}^{n} L(o_i,\theta)$.
Our SNAP model is simply weighting each item with the associated agreement weight:
\begin{equation}
\min_{\theta \in \Theta} F(\theta) = \sum_{i=1}^{n} w_i \cdot L(o_i,\theta)
\label{eqn:snap_model}
\end{equation}
This model downweights outliers naturally. It is a task-agnostic, self-consistent influence weighting; further extension to task-dependent weighting $w_i({\theta})$ will be discussed in Section \ref{sec:discussion}. The optimization problem in \eqref{eqn:snap_model} may have a closed-form solution or be solved iteratively. 

Robust estimation methods traditionally aim to limit the influence of outliers by modifying the loss function \cite{Maronna2019}. SNAP adopts a complementary perspective by explicitly controlling the \emph{relative} contribution of samples through adaptive weights. Because parameter updates are determined by relative gradient contributions, explicit relative downweighting helps ensure that inliers dominate optimization even in the presence of outliers, without relying on loss-specific scale tuning. Overall, SNAP shifts the focus from bounding the absolute influence of individual samples to explicitly regulating their relative impact on optimization. This perspective complements existing robust estimation methods.

The key contribution of this paper is the SNAP weighting scheme and its theoretical analysis. A comprehensive study of \eqref{eqn:snap_model} and comparison with state-of-the-art robust estimation methods is beyond the scope of this work and left for future research. Instead, we focus on two concrete instances, aggregation and subspace estimation, to validate the effectiveness of SNAP weighting.

\subsection{SNAP Model for Robust Aggregation}

Given the weights $w_i$, SNAP defines a \emph{weighted consensus} estimate for aggregation. In a general (not necessarily metric) space $(\mathcal{X}, d)$, a general framework of consensus of a set of given objects $o_1, \dots, o_n \in \mathcal{X}$ is defined by:
\begin{equation}
    \bar{o} \ = \ \arg\min_{x \in \mathcal{X}} \sum_{i=1}^n w_i \cdot d(x, o_i)
\label{eqn:GM}
\end{equation}
which is also called generalized median \cite{Nienkotter23}. This framework has been intensively studied for numerous problem domains and found numerous applications in many fields, see Appendix \ref{app:GM} for details.

\subsection{SNAP Model for Robust Subspace Estimation}

Here we consider the linear case for simplicity reasons. Let $\vector{x}_1,\dots,\vector{x}_n \in \R^m$ be data points, and let $\vector{U} \in \R^{m \times k}$ be a matrix with orthonormal columns ($\vector{U}^\top \vector{U} = \vector{I}_k$) defining a $k$-dimensional linear subspace. Then the general subspace estimation problem, extended with the SNAP weights, is defined by:
\begin{equation}
\min_{\vector{U} \in \R^{m \times k},\, \vector{U}^\top \vector{U} = \vector{I}_k} \sum_{i=1}^n w_i \cdot f\big(\|\vector{x}_i - P_{\vector{U}} \vector{x}_i\|\big)
\end{equation}
where $P_\vector{U} \vector{x}_i = \vector{U} \vector{U}^\top \vector{x}_i$
is the orthogonal projection of $\vector{x}_i$ onto the subspace, and $f: R_{\ge 0} \to R_{\ge 0}$ is a loss function that encodes the desired criteria. The special case $f(t)=t^2$ leads to the popular PCA.

\subsection{Analysis of SNAP ERM in \eqref{eqn:snap_model}}

We begin with a high-level analysis of robustness and efficiency in arbitrary spaces, which can then be further deepened for specific spaces.

\noindent
\textbf{Robustness Analysis.}
Intuitively, SNAP assigns influence to each data item proportional to how strongly it agrees with the rest of the data set, resulting in a soft, self-supervised determination of weights. This model downweights outliers naturally.

\begin{theorem}[Deterministic Finite-Sample Bound]
\label{thm:bound_erm}
For any fixed $\theta \in \Theta$, the weighted empirical risk satisfies:
\[
|F(\theta) - F_{\mathcal{I}}(\theta)| = \sum_{i \in \Out} w_i \cdot L(o_i,\theta)
\leq w_{\max}^{\mathrm{out}} \sum_{i \in \Out} L(o_i, \theta)
\]
where $\mathcal{I}$ is the inlier set and $w_{\max}^{\mathrm{out}} = \max_{i \in \Out} w_i$ is the largest outlier weight.
\end{theorem}
Outliers receive very small weights, with even $w_{\max}^{\mathrm{out}}$ being negligible. The key result of SNAP weighting in Theorem~\ref{thm:key_result} demonstrates exponential suppression of outlier weights, tightly bounding $w_{\max}^{\mathrm{out}}$ by $e^{-cn}$ or $e^{-cn^2}$.

\noindent
\textbf{Efficiency Analysis.} 
In SNAP, if no outliers are present, all inliers receive substantial weights, contributing fully to the computation and achieving statistical efficiency. In this clean regime, $w_{\max}^{\mathrm{out}} = 0$, showing that robustness does not compromise accuracy on uncontaminated data.

\section{Applications of SNAP}

We demonstrate SNAP on three tasks. The last one on moving average demonstrates its extension to temporal data.

\subsection{Averaging in $\R^m$}
\label{sec:Rm}

Vector averaging is a fundamental aggregation operation, providing a principled way to combine multiple vector-valued observations into a single representative quantity. It is widely used across many fields. We discuss it in detail because it underpins a key result of this paper with practical relevance and broad applicability.

Given $\mathcal{X} = \mathbb{R}^m$, two popular distances are the $\ell_2$ norm and its squared version. Using the squared $\ell_2$ norm, optimization in \eqref{eqn:GM} yields the arithmetic average (Fr\'echet mean), which is well-known to be non-robust, with a breakdown point of zero. In contrast, optimization with the $\ell_2$ norm has no closed-form solution; the popular Weiszfeld algorithm \cite{beck_weiszfelds_2015} solves it iteratively. This $\ell_2$-based solution, known as the geometric median, is robust with a breakdown point of 0.5. These robustness results follow from the general framework in \cite{Nienkötter2025}, which shows that metric distance functions (e.g. $\ell_2$) guarantee a breakdown point $\geq 0.5$, whereas non-metric distances (e.g. $\ell_2^2$) can yield a breakdown point of zero. However, even the robust $\ell_2$ solution is influenced by outliers and not entirely immune to their effect.

Using the weights $w_i$, SNAP defines the robust vector averaging as:
\begin{eqnarray}
\ell_2: & \bar{\vector{x}} = \arg \displaystyle\min_{\bar{\vector{x}}} \displaystyle\sum_{i=1}^n w_i \| \vector{x}_i - \bar{\vector{x}} \| & \label{eqn:vector_averaging} \\
\ell_2^2: & \bar{\vector{x}} = \arg \displaystyle\min_{\bar{\vector{x}}} \displaystyle\sum_{i=1}^n w_i \| \vector{x}_i - \bar{\vector{x}}\|^2 = \sum_{i=1}^n w_i \vector{x}_i & \label{eqn:vector_averaging_squared}
\end{eqnarray}
\textbf{Robustness and Efficiency Analysis.}
For $\ell_2$, there is no closed-form solution for \eqref{eqn:vector_averaging}, so we apply the weighted Weiszfeld algorithm. Thanks to the properties of SNAP weights, the influence of outliers is expected to be strongly limited.

The non-iterative solution \eqref{eqn:vector_averaging_squared} for $\ell_2^2$ is particularly interesting. Without weighting, it reduces to the arithmetic average and is not robust. As we show in Section~\ref{sec:exp_Rm}, SNAP weighting suppresses outliers so effectively that this solution outperforms both the standard Weiszfeld algorithm ($\ell_2$) and its SNAP-weighted variant in \eqref{eqn:vector_averaging}.

The following theorem (see Appendix \ref{app:gaussian} for proof) indicates that the agreement weights in \eqref{eqn:agreement_weight} strongly support the average computation in general.
%
\begin{theorem}[restate=theoremGaussian, name={}]
We use $\ell_2^2$ and $\kappa(x)=e^{-x/\sigma}$ or $\kappa(x)=e^{-x^2/\sigma^2}$ for the agreement weights in $\R^m$. If the data are Gaussian distributed, then $\bar{\vector{x}}$ in \eqref{eqn:vector_averaging_squared} exactly computes the true average.
\label{theorem:gaussian}
\end{theorem}

Now we derive and discuss two further results for the non-iterative solution \eqref{eqn:vector_averaging_squared}. Given the inlier set $\I$, let $\vector{\mu}_\I$ denote the sample mean of inliers, which represents the optimal computation result if we were able to reliably detect all outliers and can be used for the purpose of performance measurement. We provide two deterministic finite-sample bounds on the deviation between our weighted estimator and the inlier mean.
\begin{theorem}[restate=theoremBoundA, name=Deterministic Finite-Sample Bound]
\label{thm:bound_1}
For any weights $w_i$, the weighted estimator \eqref{eqn:vector_averaging_squared} satisfies:
\[
\label{eqn:deterministic_bound}
\|\bar{\vector{x}} - \vector{\mu}_\I\|
\le
w_{\max} \sum_{i\in \I} \|\vector{x}_i - \vector{\mu}_\I\| + \sum_{i \in \Out} w_i \|\vector{x}_i - \vector{\mu}_\I\|
\]
where $w_{\max}^{\mathrm{in}} = \max_{i \in \I} w_i$ is the largest inlier weight.
\end{theorem}
See Appendix \ref{app:bound_1} for proof. The second term measures the influence of outliers, which is automatically downweighted by SNAP weighting. This fact supports the intuitive understanding and practical working of the SNAP framework.

Now we give a slightly different derivation for further discussion. It is easy to show:
\[
\label{eqn:deterministic_bound}
\|\bar{\vector{x}} - \boldsymbol{\mu}_\mathcal{I}\|
\ \leq \ \left\| \sum_{i\in \mathcal{I}} \left|w_i - \frac{1}{|\mathcal{I}|}\right| \vector{x}_i \right\| + \sum_{i \in \Out} w_i \|\vector{x}_i - \vector{\mu}_\mathcal{I}\| 
\]
If there are no outliers, the second term vanishes. The roughly uniform inliers receive similar weights $w_i \approx 1/|\mathcal{I}|$. Thus, the computation becomes $\bar{\vector{x}} \approx \boldsymbol{\mu}_\I = 1/|\mathcal{I}| \sum_{i \in \mathcal{I}} \vector{x}_i$, recovering classical $O(1/\sqrt{n})$ scaling. Importantly, by involving all inliear data, this estimation is statistically efficient.

\subsection{PCA}

PCA estimates a subspace capturing the directions of highest variance, often for dimensionality reduction. However, it is sensitive to outliers, as illustrated in Figure~\ref{fig:illustration} (right). By applying $\ell_2$-based SNAP weights, we focus PCA on the inliers via a weighted covariance matrix (assuming $\mathbf{x}_i$ are column vectors):
\begin{equation}
    \boldsymbol{\Sigma} \ = \ \sum_{i=1}^n w_i (\vector{x}_i - \bar{\vector{x}}) (\vector{x}_i - \bar{\vector{x}})^T
\end{equation}
PCA is then performed on this weighted covariance matrix via standard eigenvalue decomposition.

\subsection{Moving Average}

Given a time series $\{x_t\}_{t=1}^n \subset \mathbb{R}$, we extend SNAP to produce a smoothed series $\{\hat{x}_t\}_{t=1}^n$ that captures the underlying trend while reducing noise. This extension is detailed in Appendix~\ref{app:MA}, and experiments show it outperforms the standard exponentially moving average.

\section{Experiments}
\label{sec:experiments}

\subsection{Averaging in $\R^m$}
\label{sec:exp_Rm}

\begin{table}[t]
\caption{SNAP averaging in $\R^m$.
}
\label{tab:vector_experiment}
\begin{center}
    $\R^{2}$
\end{center}

\vspace{-5mm}
\begin{center}
   \begin{small}
        \begin{sc}
            \begingroup
            \setlength{\tabcolsep}{5pt}
            \begin{tabular}{ll|rrrrr}
            \toprule
            Outlier \% & & 10\% & 20\% & 30\% & 40\% & 49\% \\
            \midrule
            \midrule
            \textit{Inlier Mean} &  & \textit{0.095} & \textit{0.102} & \textit{0.104} & \textit{0.110} & \textit{0.121} \\
            \midrule
            Weiszfeld &  & 0.190 & 0.379 & 0.644 & 1.052 & 1.677 \\
            \midrule
            \multirow[c]{2}{*}{SNAP $\ell_2$ $\kappa_l$} & $\bmad$ & 0.162 & 0.244 & 0.334 & 0.484 & 0.787 \\
             & $\bmed$ & 0.152 & 0.223 & 0.313 & 0.475 & 0.804 \\
            \midrule
            \multirow[c]{2}{*}{SNAP $\ell_2$ $\kappa_g$} & $\sigmamad$ & 0.139 & \textbf{0.182} & \textbf{0.228} & 0.390 & 0.756 \\
             & $\sigmamed$ & 0.142 & 0.191 & 0.244 & \textbf{0.373} & \textbf{0.711} \\
            \midrule
            \multirow[c]{2}{*}{SNAP $\ell_2^2$ $\kappa_l$} & $\bmad$ & 0.150 & 0.248 & 0.382 & 0.731 & 1.408 \\
             & $\bmed$ & 0.140 & 0.232 & 0.384 & 0.801 & 1.599 \\
            \midrule
            \multirow[c]{2}{*}{SNAP $\ell_2^2$ $\kappa_g$} & $\sigmamad$ & \textbf{0.134} & 0.196 & 0.314 & 0.815 & 1.693 \\
             & $\sigmamed$ & 0.137 & 0.207 & 0.310 & 0.680 & 1.523 \\
            \bottomrule
            \end{tabular}
            \endgroup
        \end{sc}
    \end{small}
\end{center}

\begin{center}
    $\R^{50}$
\end{center}

\vspace{-5mm}
\begin{center}
    \begin{small}
        \begin{sc}
            \begingroup
            \setlength{\tabcolsep}{5pt}
            \begin{tabular}{ll|rrrrr}
            \toprule
            Outlier \% & & 10\% & 20\% & 30\% & 40\% & 49\% \\
            \midrule
            \midrule
            \textit{Inlier Mean} &  & \textit{0.519} & \textit{0.549} & \textit{0.591} & \textit{0.635} & \textit{0.682} \\
            \midrule
            Weiszfeld &  & 0.846 & 1.637 & 2.812 & 4.741 & 8.424 \\
            \midrule
            \multirow[c]{2}{*}{SNAP $\ell_2$ $\kappa_l$} & $\bmad$ & 0.737 & 0.745 & 0.744 & 0.724 & 0.800 \\
            & $\bmed$ & 0.557 & 0.590 & 0.633 & 0.675 & 0.833 \\
            \midrule
            \multirow[c]{2}{*}{SNAP $\ell_2$ $\kappa_g$} & $\sigmamad$ & 0.637 & 0.640 & 0.641 & 0.652 & 0.911 \\
            & $\sigmamed$ & 0.540 & 0.569 & 0.609 & 0.648 & \textbf{0.699} \\
            \midrule
            \multirow[c]{2}{*}{SNAP $\ell_2^2$ $\kappa_l$} & $\bmad$ & 0.687 & 0.705 & 0.715 & 0.704 & 1.758 \\
            & $\bmed$ & 0.540 & 0.574 & 0.618 & 0.677 & 2.107 \\
            \midrule
            \multirow[c]{2}{*}{SNAP $\ell_2^2$ $\kappa_g$} & $\sigmamad$ & 0.614 & 0.621 & 0.627 & 0.642 & 3.039 \\
            & $\sigmamed$ & \textbf{0.529} & \textbf{0.559} & \textbf{0.599} & \textbf{0.639} & 0.894 \\
            \bottomrule
            \end{tabular}
            \endgroup
        \end{sc}
    \end{small}
\end{center}
\end{table}

\begin{figure*}[t]
    \centering
    \begin{tikzpicture}[node distance=0pt, inner sep=0pt]
        \node (a) {\includegraphics[width=0.25\linewidth]{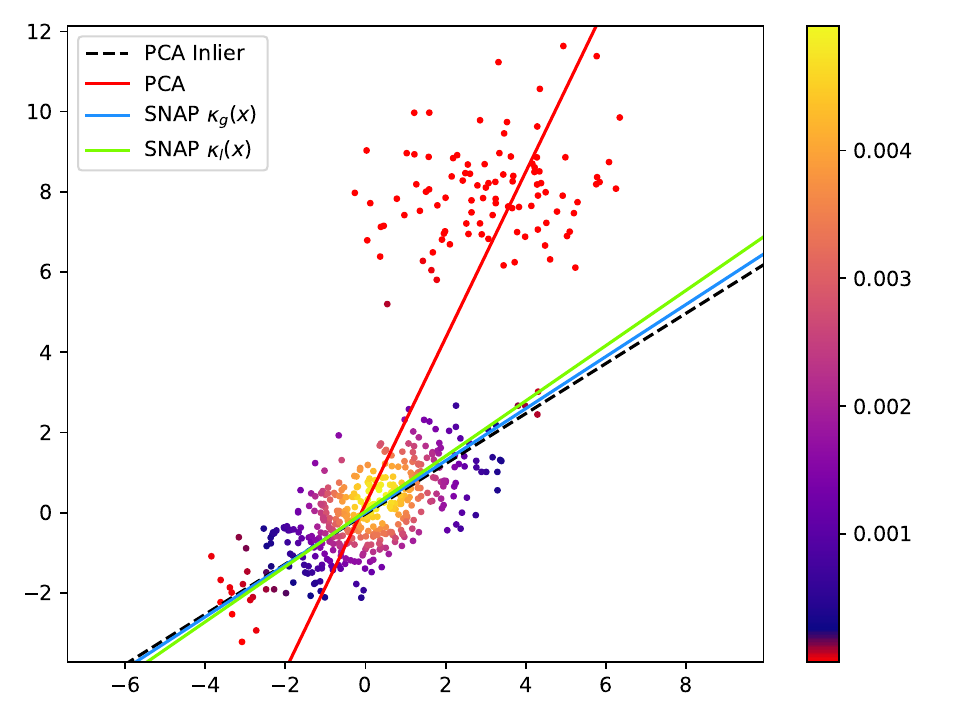}};
        \node[anchor=center] at ($(a.north west)+(0.6,-1.1)$) {\large a)}; 
        \node[right=of a] (b) {\includegraphics[width=0.25\linewidth]{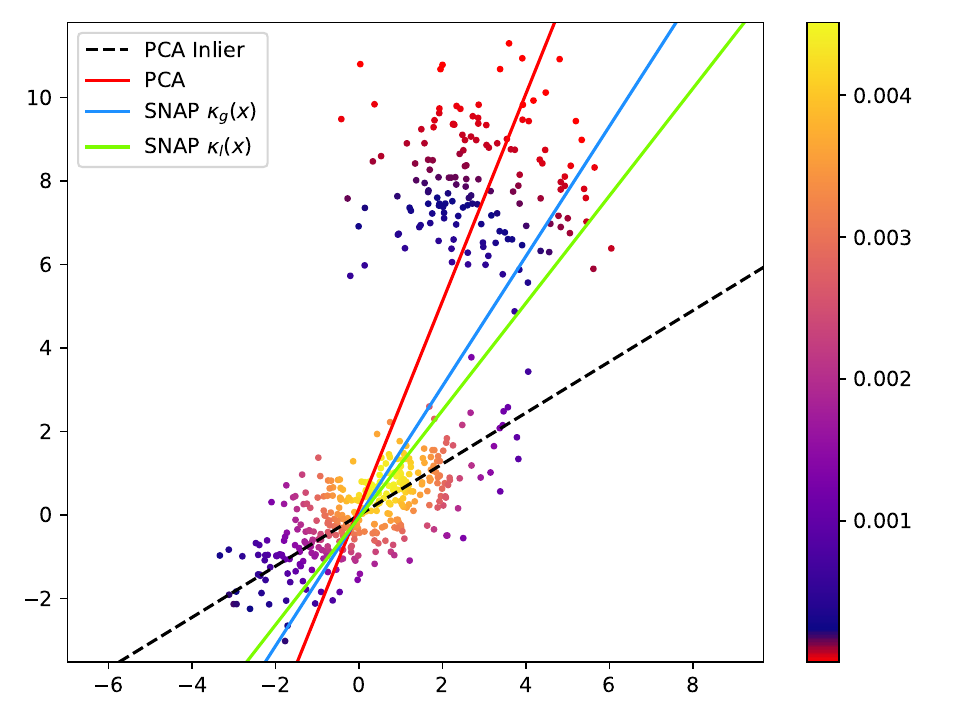}};
        \node[anchor=center] at ($(b.north west)+(0.6,-1.1)$) {\large b)}; 
        \node[right=of b] (c) {\includegraphics[width=0.25\linewidth]{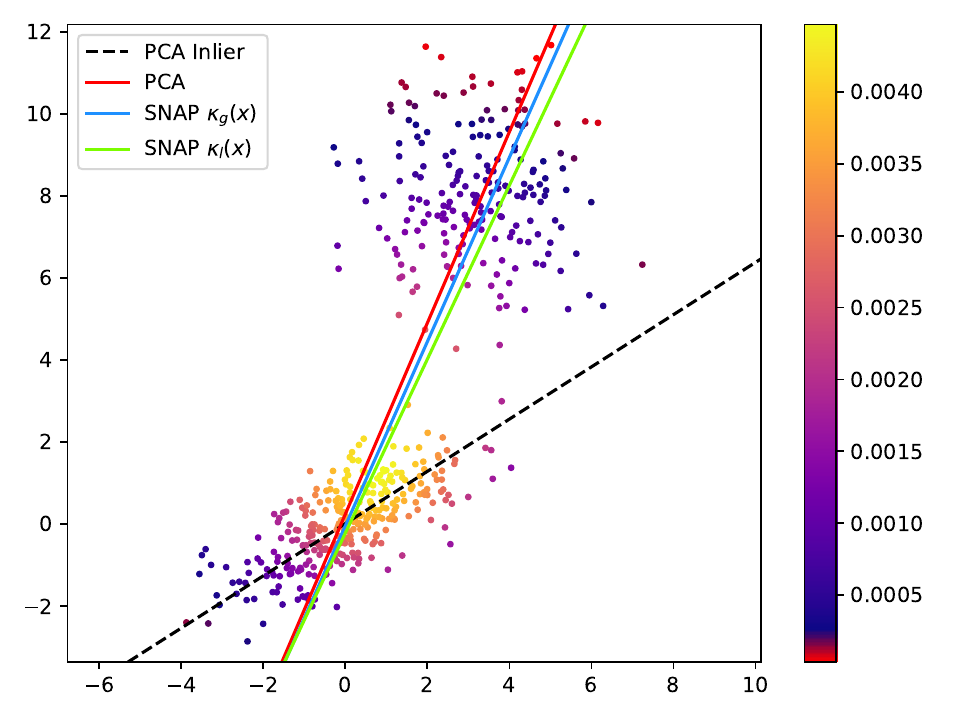}};
        \node[anchor=center] at ($(c.north west)+(0.6,-1.1)$) {\large c)}; 
        \node[right=of c] (d) {\includegraphics[width=0.25\linewidth]{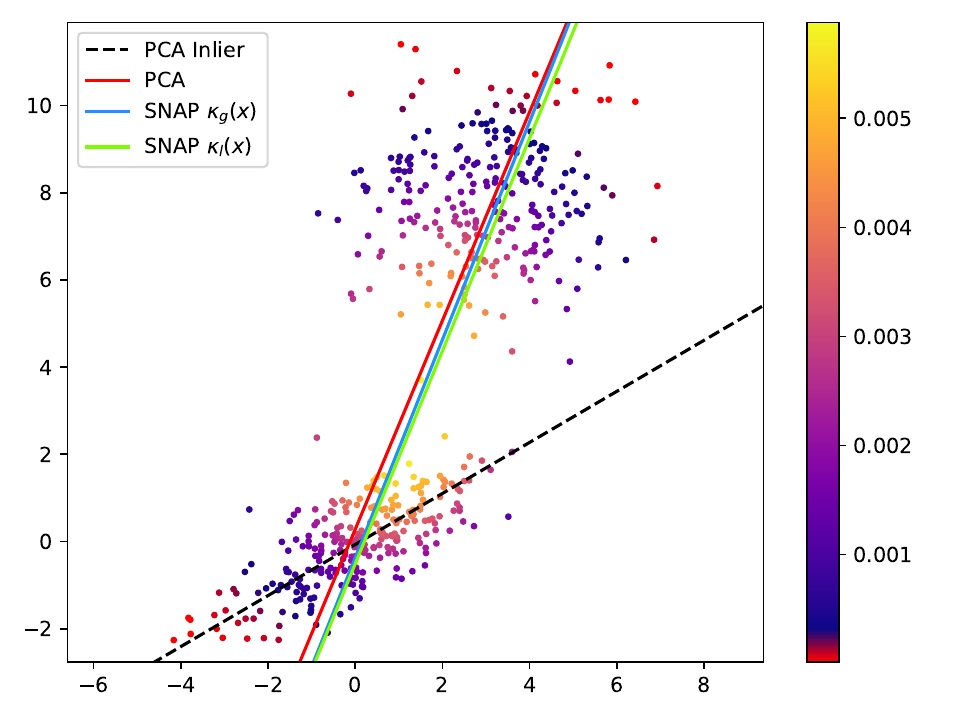}};
        \node[anchor=center] at ($(d.north west)+(0.6,-1.1)$) {\large d)}; 
    \end{tikzpicture}
    \caption{Example results for PCA estimation, with and without SNAP weights. a): 20\% outliers; b) 30\% outliers; c) 40\% outliers; d) 49\% outliers. The result for 10\% outliers is shown in Figure \ref{fig:illustration} (right).}
    \label{fig:pca_experiments}
\end{figure*}
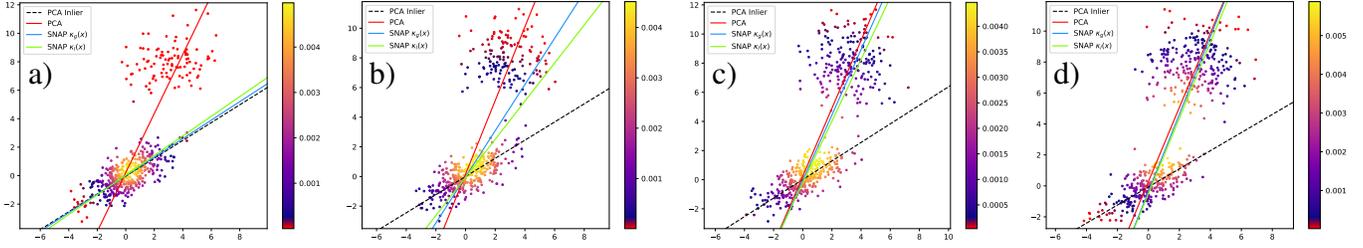

\begin{table*}[t]
\caption{SNAP PCA (all angle errors in degree).}\label{tab:pca_experiment}
\begin{center}
\begin{sc}
\begin{tabular}{ll|rr|rr|rr|rr|rr}
\toprule
Outlier \% & & \multicolumn{2}{c|}{10\%} & \multicolumn{2}{c|}{20\%} & \multicolumn{2}{c|}{30\%} & \multicolumn{2}{c|}{40\%} & \multicolumn{2}{c}{49\%} \\
Error & & Angle & Proj & Angle & Proj & Angle & Proj & Angle & Proj & Angle & Proj \\
\midrule
\midrule
PCA & & 30.037 & 0.708 & 33.542 & 0.781 & 35.000 & 0.811 & 35.682 & 0.824 & \textbf{36.109} & \textbf{0.833} \\
\midrule
\multirow[c]{2}{*}{SNAP $\kappa_l(\DA)$} & $\bmad$ & 1.283 & 0.032 & 2.619 & 0.065 & 19.818 & 0.479 & \textbf{34.127} & \textbf{0.793} & 36.531 & 0.841 \\
 & $\bmed$ & 1.024 & 0.025 & 3.749 & 0.092 & 23.255 & 0.558 & 35.009 & 0.811 & 36.792 & 0.847 \\
\midrule
\multirow[c]{2}{*}{SNAP $\kappa_g(\DA)$} & $\sigmamad$ & \textbf{0.945} & \textbf{0.023} & \textbf{0.941} & \textbf{0.023} & 23.265 & 0.558 & 35.545 & 0.822 & 36.838 & 0.847 \\
 & $\sigmamed$ & 1.100 & 0.027 & 1.010 & 0.025 & \textbf{12.676} & \textbf{0.310} & 35.269 & 0.816 & 36.905 & 0.849 \\
\bottomrule
\end{tabular}
\end{sc}
\end{center}
\end{table*}

Table~\ref{tab:vector_experiment} shows the mean absolute error results over 100 trials for SNAP with kernel functions $\kappa_g(\DA)$ and $\kappa_l(\DA)$, each using both variants of robust $b$ and $\sigma$ estimation. SNAP $\ell_2$ and SNAP $\ell_2^2$ implement robust averaging according to \eqref{eqn:vector_averaging} and \eqref{eqn:vector_averaging_squared}, respectively. The unweighted Weiszfeld algorithm is included as a baseline, and the Inlier Mean provides a lower bound on error. Datasets in $\mathbb{R}^2$ and $\mathbb{R}^{50}$ are generated as follows: first, an inlier set $\I$ is drawn from a Gaussian distribution with $\mu=\mathbf{0}$ and $\sigma=\mathbf{I}_m$; then, uniformly distributed outliers in $[0,10]$ are added in all dimensions, for a total of 200 points per set. An example $\mathbb{R}^2$ dataset with 30\% outliers is shown in Figure~\ref{fig:illustration} (left).

All SNAP variants consistently outperform Weiszfeld, achieving much lower average errors. SNAP $\ell_2$ with $\kappa_g(\DA)$ performs best in $\mathbb{R}^2$, while SNAP $\ell_2^2$ with $\kappa_g(\DA)$ achieves the best overall results in $\mathbb{R}^{50}$. This demonstrates that SNAP outperforms the robust iterative Weiszfeld algorithm even using the simple non-iterative weighted mean in \eqref{eqn:vector_averaging_squared}.

In high dimensions, absolute differences shrink, which challenges Weiszfeld (Table~\ref{tab:vector_experiment}). However, relative distance ordering is often preserved, which underlies SNAP weighting and explains its robust performance in $\mathbb{R}^{50}$.

Additional comparison with component-wise median and two multivariate variants of median of means \cite{Hogsgaard2025} is shown in Appendix \ref{app:weiszfeld}. Our SNAP results outperform all compared methods, particularly in high dimensions.

\subsection{PCA}

Table~\ref{tab:pca_experiment} reports the average performance of SNAP compared to standard PCA. We generated 100 datasets with 500 points each in $\mathbb{R}^2$: inliers are Gaussian with $\boldsymbol{\mu}=\mathbf{0}$ and $\boldsymbol{\Sigma}=[2,1;1,1]$, while outliers are Gaussian with $\boldsymbol{\mu}=[3,8]$ and $\boldsymbol{\Sigma}=2\mathbf{I}_2$. Methods are compared by projecting to one dimension using: (1) \textbf{Angle error} -- the difference in degrees between the ground truth and computed first principal component; and (2) \textbf{Projection error} -- the Frobenius norm of the difference between the projection matrices $\mathbf{U}\mathbf{U}^\top$ of standard PCA on the inliers and the method under evaluation. Example results for 10\% outliers are shown in Figure~\ref{fig:illustration} (right), and results for 20--49\% outliers are in Figure~\ref{fig:pca_experiments}. Full error tables are given in Table~\ref{tab:pca_experiment}.  

SNAP effectively handles outliers, reconstructing inlier principal components with substantially lower error than standard PCA. In most cases, $\kappa_g(\DA)$ with $\sigmamad$ performs best, with other SNAP variants also performing well. Errors increase sharply beyond 30\% outliers as the outlier set forms a separate cluster, violating the ARH assumption (Section~\ref{sec:arh}) and receiving higher SNAP weights. In this regime, a common approach is to detect multiple clusters and apply a robust estimator locally; SNAP fits naturally into this pipeline.

\section{Discussions}
\label{sec:discussion}

\subsection{Potential of SNAP Robust Computation}

\noindent
\textbf{Robustifying non-metric consensus computation.}
Consensus methods sometimes rely on non-metric or dissimilarity measures that are efficient but sensitive to outliers. SNAP weighting addresses this by emphasizing mutually consistent inliers and suppressing dispersed outliers, providing substantial robustness with minimal cost. As shown in Section~\ref{sec:Rm}, SNAP even outperforms the iterative Weiszfeld algorithm when applied to the simple non-iterative weighted mean in \eqref{eqn:vector_averaging_squared} using the non-metric $\ell_2^2$.

\noindent
\textbf{SNAP as guided initialization.}
Although SNAP-based weighting may not achieve the global optimum, it provides a strong self-supervised initialization for iterative robust estimation. By emphasizing consensus and suppressing outliers, SNAP reduces the effective search space and improves convergence, serving both as a standalone robust method and an efficient preconditioning strategy.

\noindent
\textbf{Working with high dimension.}
In high-dimensional spaces, where absolute differences shrink, many methods struggle, a limitation well known for the Weiszfeld algorithm (Section~\ref{sec:Rm}) and others. SNAP differs fundamentally by relying on the relative ordering of distances, which is often preserved even in high dimensions, making it particularly effective for high-dimensional data.

\subsection{Limitations and Extensions}

The current weight computation in \eqref{eqn:agreement_weight} has quadratic complexity $O(n^2)$. While Section~\ref{sec:scalability} demonstrates that SNAP can handle moderately sized datasets, addressing scalability for truly large datasets with appropriate approximation techniques remains an important task. Moreover, the SNAP model in \eqref{eqn:snap_model} currently employs task-independent weighting; extending it to task-dependent schemes, particularly for applications such as function fitting, is a promising problem. Finally, although our experiments focus on continuous spaces, SNAP naturally extends to discrete domains, which we leave for future investigation.


\section{Conclusion}


We introduce \textbf{SNAP} (Self-Consistent Agreement Principle), a general framework for self-supervised weighting based on mutual agreement. Decoupled from specific applications, SNAP provides a flexible, easy-to-use, broadly applicable tool for robust estimation, bridging theoretical rigor with practical effectiveness. Its agreement weights are fully parameter-free and model-agnostic, unlike many robust methods, and the framework applies to arbitrary spaces, including non-numerical domains. We envision SNAP as the foundation for a new framework for robust computation; this paper establishes its theoretical underpinnings and presents initial applications.




\section*{Acknowledgments}

This work was supported by the Deutsche Forschungsgemeinschaft (DFG: CRC 1450 – 431460824, SPP2363, to X.J.), the National Natural Science Foundation of China (Grant Nr. W2433165, to A.N.), and the Key Research and Development Program of Sichuan Province (Grant Nr. 2023YFWZ0009, to A.N.).%

\bibliography{References}
\bibliographystyle{IEEEtran}

\newpage

\onecolumn

\appendix

\subsection{Related Work on Robust Estimation, Geometry, and Agreement Principles}
\label{app:related_work}

Robust estimation has a long history in statistics and optimization, with classical approaches focusing on bounding the influence of corrupted observations through modified loss functions or constrained estimators. In structured geometric settings, this line of work has led to strong theoretical guarantees. For example, Maunu and Lerman~\cite{maunu2019robust} study robust subspace recovery under adversarial outliers and provide recovery guarantees under precise structural assumptions, while Yang and Carlone~\cite{yang2021certifiably} develop certifiably optimal and scalable methods for outlier-robust geometric perception via semidefinite relaxations. These methods are tailored to specific geometric objectives such as subspace estimation, robust consensus, or geometric median computation, and focus on global optimality and recovery guarantees. In contrast, SNAP employs a weighted scheme for robust computation. This paper focuses on a parameter-free, data-driven weighting approach that downweights outliers, thereby enhancing robustness

Beyond classical robustness, recent work has emphasized the role of agreement, consistency, and emergent structure as organizing principles for learning and decision-making. A growing body of work in social choice, alignment, and collective decision-making studies robustness through agreement-based mechanisms rather than loss shaping. Examples include pairwise calibrated rewards for pluralistic alignment~\cite{halpern2025pairwise}, welfare and representation guarantees in consensus mechanisms~\cite{deraaij2025representation}, and generative approaches to social choice~\cite{boehmer2025generative}. These works provide formal frameworks in which robustness arises from mutual consistency or agreement among agents, rather than from individual pointwise penalties.

Related perspectives also arise in proximity-based and alignment-oriented frameworks. Shirali et al.~\cite{shirali2025direct} study direct alignment under heterogeneous preferences. Applications of agreement principles to robust selection and deliberation further illustrate the breadth of this paradigm~\cite{assos2025alternates}. While SNAP does not explicitly instantiate these agreement or proximity mechanisms, such frameworks suggest promising future directions.

Taken together, these lines of work highlight complementary approaches to robustness: classical robust estimation emphasizes bounded influence and recovery guarantees for specific objectives, geometric methods focus on certifiable optimality under structured assumptions, and agreement-based frameworks study robustness as an emergent property of mutual consistency. SNAP contributes to this landscape by providing a lightweight, weight-based approach to robust computation.

\subsection{Properties of Agreement Weights}


\subsubsection{Proof of Proposition \ref{prop:sensitivity}}
\label{app:sensitivity}

\propSensitivity*

\begin{proof}
Derivative of $w_i$ with respect to $\DA_i$:
\[
    \frac{\partial w_i}{\partial \DA_i} \ = \ \frac{k'(\DA_i) K - \kappa(\DA_i) \frac{\partial K}{\partial \DA_i}}{K^2}
    \ = \ \frac{k'(\DA_i) K - \kappa(\DA_i) \kappa'(\DA_i)}{K^2} \\
    \ = \ \frac{\kappa'(\DA_i) (K - \kappa(\DA_i))}{K^2}
\]
Since $\kappa(\DA)$ is a decreasing kernel, $\kappa'(\DA_i) < 0$ and $K - \kappa(\DA_i) > 0$, therefore $\partial w_i / \partial \DA_i < 0$.

Derivative of $w_i$ with respect to $\DA_j$, $j \neq i$:
\[
\frac{\partial w_i}{\partial \DA_j} = \frac{0 \cdot K - \kappa(\DA_i) \kappa'(\DA_j)}{K^2} = -\frac{\kappa(\DA_i) \kappa'(\DA_j)}{K^2}
\]
Since $\kappa'(\DA_j) < 0$, we have $\partial w_i / \partial \DA_j > 0$.  
\end{proof}

\subsubsection{Proof of Proposition \ref{prop:disagreement_sensitivity}}
\label{app:disagreement_sensitivity}

\propDisagreementSensitivity*

\begin{proof}
The disagreement score of $o_k$ is computed by:
\[ \DA_k \ = \ \frac{D_k}{2 D_k + \sum_{i\not=k, \ j\not=k} d(o_i, o_j)} \]
The way how $o_k$ is perturbed to generate $o_k'$ means $D'_k > D_k$, which results in $\DA'_k > \DA_k$. According to Proposition \ref{prop:sensitivity} we thus obtain $w'_k < w_k$.

Note that this result means that when fixing all elements except $o_k$ and letting
$d(o_k,o_{i \not= k}) \to\infty$, $w_k$ decreases and approaches to some limit. Proposition \ref{prop:outlier_amplification} further specifies this limit to be zero if $n\to\infty$.
\end{proof}

\subsubsection{Proof of Proposition \ref{prop:global_bound}}
\label{app:global_bound}

\propGlobalBound*

\begin{proof}
By construction, $\vector{w}$ and $\vector{w}^\prime$ lie in the probability simplex (although we do not interpret the values as probabilities). Hence,
\[
\|\vector{w}\|^2 \ = \ \sum_i \vector{w}_i^2 \ \le \ \sum_i \vector{w}_i \ = \ 1
\]
and similarly $\|\vector{w}^\prime\| \le 1$.
Therefore,
\[
\|\vector{w}-\vector{w}^\prime\|^2
\ = \ \|\vector{w}\|^2 + \|\vector{w}^\prime\|^2 - 2\langle \vector{w},\vector{w}^\prime\rangle
\ \le \ 2
\]
which implies $\|\vector{w}-\vector{w}^\prime\| \le \sqrt{2}$.
\end{proof}

\subsubsection{Proof of Proposition \ref{prop:outlier_amplification}}
\label{app:outlier_amplification}

\propOutlierAmplification*

\begin{proof}
As $d(x_k,x_{i \not= k}) \to\infty$,
\[
D_k = \sum_{j \neq k} d(x_k,x_j) \approx (n-1) d(x_k,x_l)
\]
\[
D_i = \sum_{j \neq i} d(x_i,x_j) \approx d(x_k,x_i), \mbox{ for any } i \neq k
\]
where $l$ is any index $\not=k$. Thus,
\[
D_a \ = \ \sum_{p=1}^n D_p \ \approx \ (n-1)d(x_k,x_l) + \sum_{i\not=k} d(x_k,x_i)
\ \approx \ 2(n-1) d(x_k,x_l)
\]
Then, the normalized disagreement scores become:
\[ \DA_k \ = \ \frac{D_k}{D_a} \ \approx \ \frac{1}{2}, \quad \DA_{i\not=k} \ = \ \frac{D_i}{D_a} \ \approx \ \frac{1}{2(n-1)} \ \to 0 \ (\mbox{if } n \to\infty) \]
Finally, the agreement weight for the extreme outlier $o_k$:
\[ w_k \ \approx \ \frac{\kappa(\frac{1}{2})}{\kappa(\frac{1}{2})+(n-1)\kappa(0)} \ \to \ 0 \]
\end{proof}

\subsubsection{Proof of Theorem \ref{thm:key_result}}
\label{app:key_result}

\Suppression*

\begin{proof}
We fix $k \in \Out$ and select $i \in \I$ with largest $\Delta_i$ among all inliers. By Assumption~\ref{ass:linear_gap} we have:
\[
\sum_{j=1}^n d(o_k,o_j) - \sum_{j=1}^n d(o_i,o_j) \ \ge \ c_0 n
\]
Dividing by the total sum $D_a = \sum_{p,q} d(o_p,o_q)$ gives:
\[
\Delta_k - \Delta_i \ \ge \ \frac{c_0 n}{D_a} \ = \ \frac{n}{c_1}
\]
since $D_a \le C n^2$ for some constant $C$.

\noindent
{\bf Case 1:} Kernel function $\kappa(x)=e^{-x/\sigma}$
\[
\frac{\kappa(\Delta_k)}{\kappa(\Delta_i)}
\ = \
\exp\!\left(-\frac{\Delta_k-\Delta_i}{\sigma}\right)
\ \le \
\exp\!\left(-\frac{n}{\sigma c_1}\right)
\]

Since there are $\alpha n$ inliers and $\kappa(\Delta_i)$ is minimized over $j\in \I$, we have:
\[
\sum_{j=1}^n \kappa(\Delta_j)
\ \ge \
\sum_{i\in \I} \kappa(\Delta_i)
\ \ge \
\alpha n \, \kappa(\Delta_i)
\]
Thus, for any outlier $k \in \Out$:
\[
w_k
=
\frac{\kappa(\Delta_k)}{\sum_j \kappa(\Delta_j)}
\ \le \
\frac{1}{\alpha n}
\frac{\kappa(\Delta_k)}{\kappa(\Delta_i)}
\ \le \
\frac{1}{\alpha n}
\exp\!\left(-\frac{n}{\sigma c_1}\right)
\ = \
\frac{1}{\alpha n} e^{-cn}
\]
The exponential term dominates the polynomial prefactor as $n\to\infty$.
Hence there exist constants $c>0$ and $n_0$ such that for all $n\ge n_0$:
\[
w_k \;\le\; e^{-c n}
\]

\noindent
{\bf Case 2:} Ke2nel function $k(x) = e^{-x^2/\sigma^2}$
\[
\Delta_k^2 - \Delta_i^2 \ = \ (\Delta_k - \Delta_i)(\Delta_k + \Delta_i) \ \ge \ c_2 n^2
\]
Thus, we obtain:
\[
w_k \ \le \ \frac{1}{\alpha n} \exp\Bigl(-\frac{\Delta_k^2 - \Delta_i^2}{\sigma^2}\Bigr)
\ \le \ \frac{1}{\alpha n} e^{-c n^2}
\]
Similarly, there exist constants $c>0$ and $n_0$ such that for all $n\ge n_0$:
\[
w_k \ \le \ e^{-c n^2}
\]
\end{proof}

\subsection{Applications of Generalized Median Computation}
\label{app:GM}

Generalized median has been intensively studied for numerous problem domains. Examples include rankings \cite{Boulakia_2011}, phase/orientation data smoothing \cite{Storath_2018,Guo2020}, 3D rotations \cite{lee_civera_robust_2025}, shapes \cite{Cunha_2019}, image segmentation \cite{Khelifi_2017,Ma2021}, clusterings \cite{Boongoen2018},  biclustering \cite{Yin2018}, graphs \cite{jiang2001,Blumenthal2021}, probability vectors (for ensemble classification) \cite{Xue2026}, anatomical atlas construction \cite{Xie2013}, averaging in Grassmann manifolds \cite{Chakraborty2021} and flag manifolds \cite{Mankovich_2023}.

Generalized median computation has found numerous applications in many fields, including machine learning (e.g. nonlinear subspace learning \cite{Chakraborty2021}), computer vision \cite{lee_civera_robust_2025}, test time augmentation \cite{Shanmugam2021}, bioinformatics \cite{Yin2018}, social sciences, computer-assisted surgery \cite{Tu2022}, and many others.


\subsection{Proofs of Theorems for Averaging in $\R^m$}

\subsubsection{Proof of Theorem \ref{theorem:gaussian}}
\label{app:gaussian}

\theoremGaussian*

\begin{proof}
This result is known in the statistics and the proof is shown for the reason of completeness only. We consider $m$-dimensional Gaussian distribution with mean $\boldsymbol{\mu} $ and covariance matrix $\boldsymbol{\Sigma}$.
Using squared Euclidean distance the numerator and denominator in \eqref{eqn:disagreement_score} for some $\vector{x}$ are:
\[
D_{\vector{x}} \ = \ \int ||\vector{x}-\vector{y}||^2 \mathcal{N}(\vector{y}; \boldsymbol{\mu},\vector{\Sigma}) \ d\vector{y} \ = \ \trace{\boldsymbol{\Sigma}} + ||\boldsymbol{\mu} -\vector{x}||^2
\]
\[
D_a \ = \ \int ||\vector{x}-\vector{y}||^2 \mathcal{N}(\vector{x}; \boldsymbol{\mu},\vector{\Sigma}) \mathcal{N}(\vector{y}; \boldsymbol{\mu} ,\boldsymbol{\Sigma}) \ d\vector{x} d\vector{y} \ = \ 2 \cdot \trace{\boldsymbol{\Sigma}}
\]
{\bf Case 1:} Kernel function $\kappa_l = e^{-x/\sigma}$ converts the disagreement into consensus measure:
\[
\kappa_l(\DA_\vector{x}) \ = \ \exp \left( - \frac{1}{\sigma} \cdot \frac{\trace{\boldsymbol{\Sigma}} + ||\vector{x}-\boldsymbol{\mu} ||^2}{2 \cdot \trace{\vector{\Sigma}}} \right) \\
\ = \ 
C \cdot \mathcal{N}(\vector{x}; \boldsymbol{\mu}, a\vector{I}_m)
\]
where $a=\sigma \cdot \trace{\vector{\Sigma}}$ and $C$ is a positive constant (also dependent on $a$). As a result, the agreement weight in \eqref{eqn:agreement_weight} becomes a Gaussian density function:
\[
w(\vector{x}) \ = \ \frac{\kappa_l(\Delta_\vector{x})}{\int \kappa_l(\Delta_\vector{x}) \ d\vector{x}} \ = \ \mathcal{N}(\vector{x}; \boldsymbol{\mu},a\vector{I}_m)
\]
Finally, the weighted consensus computation produces the mean:
\[
\int w(\vector{x}) \cdot \vector{x} \ d \vector{x} \ = \ \int \mathcal{N}(\vector{x}; \boldsymbol{\mu},a\vector{I}_m) \cdot \vector{x} \ d \vector{x} \ = \ \boldsymbol{\mu}
\]

{\bf Case 2:} Kernel function $\kappa_g = e^{-x^2/\sigma^2}$ converts the disagreement into consensus measure:
\[
\kappa_g(\DA_\vector{x})
\ = \
\exp \left( - \frac{1}{\sigma^2} \cdot \left( \frac{\trace{\boldsymbol{\Sigma}} + ||\vector{x}-\boldsymbol{\mu} ||^2}{2 \cdot \trace{\vector{\Sigma}}} \right)^2 \right)
\ = \
C \cdot \exp \left( - \left( \frac{b + ||\vector{x}-\boldsymbol{\mu} ||^2}{a} \right)^2 \right)
\]
where $a=2\sigma \cdot \trace{\vector{\Sigma}}$, $b=\trace{\boldsymbol{\Sigma}}$, and $C$ is a positive constant (also dependent on $b$). In contrast to $\kappa_l$, this is no more a Gaussian density function (subject to a normalization constant). It is also not any standard “known” distribution in $\R^m$ like Laplace, Student-t, etc. because those exponents grow linearly or quadratically, not quartically. But we can still continue with the computation. The agreement weight in \eqref{eqn:agreement_weight} becomes the density function:
\[
w(\vector{x}) \ = \ \frac{\kappa_g(\Delta_\vector{x})}{\int \kappa_g(\Delta_\vector{x}) \ d\vector{x}} \ = \ Z \cdot \exp \left( - \left( \frac{b + ||\vector{x}-\boldsymbol{\mu} ||^2}{a} \right)^2 \right),
\quad Z = \frac{C}{\int \kappa_g(\Delta_\vector{x}) \ d\vector{x}}
\]
We perform the change of variables:
\[
\vector{y} = \vector{x} - \boldsymbol{\mu},
\qquad \vector{x} = \vector{y} + \boldsymbol{\mu},
\qquad d\vector{x} = d\vector{y}
\]
Then, the weighted consensus computation is done as follows:
\begin{eqnarray*}
\int w(\vector{x}) \cdot \vector{x} \ d \vector{x}
& = &
Z \cdot \int \exp \left( - \left( \frac{b + ||\vector{x}-\boldsymbol{\mu} ||^2}{a} \right)^2 \right) \cdot \vector{x} \ d \vector{x}
\ = \
Z \cdot \int \exp \left( - \left( \frac{b + ||\vector{y}||^2}{a} \right)^2 \right) \cdot (\vector{y} + \boldsymbol{\mu}) \ d \vector{y} \\
& = &
Z \cdot \left\{\underbrace{\int \exp \left( - \left( \frac{b + ||\vector{y}||^2}{a} \right)^2 \right) \cdot \vector{y} \ d \vector{y}}_{= \ 0} + \underbrace{\int \exp \left( - \left( \frac{b + ||\vector{y}||^2}{a} \right)^2 \right) \cdot \boldsymbol{\mu} \ d \vector{y}}_{= \ 1/Z \cdot \boldsymbol{\mu}} \right\}
\end{eqnarray*}
The integrand of the first integral is odd. Since $\R^m$ is symmetric about the origin, the integral evaluates to zero. The second integral is computed by:
\[
\int \exp \left( - \left( \frac{b + ||\vector{y}||^2}{a} \right)^2 \right) \cdot \boldsymbol{\mu} \ d \vector{y}
\ = \
\boldsymbol{\mu}
\int \exp \left( - \left( \frac{b + ||\vector{x} - \boldsymbol{\mu}||^2}{a} \right)^2 \right) \ d \vector{x}
\ = \
\frac{1}{Z} \boldsymbol{\mu}
\]
Putting all together we finally obtain:
\[ \int w(\vector{x}) \cdot \vector{x} \ d \vector{x} \ = \ \boldsymbol{\mu} \]
\end{proof}

\subsubsection{Proof of Theorem \ref{thm:bound_1}}
\label{app:bound_1}

\theoremBoundA*

\begin{proof}
We first decompose the deviation from the inlier mean:  
\[
\bar{\vector{x}} - \vector{\mu}_\I \ = \ \underbrace{\sum_{i\in \I} w_i (\vector{x}_i - \vector{\mu}_\I)}_{\text{inlier contribution}} + \underbrace{\sum_{i \in \Out} w_i (\vector{x}_i - \vector{\mu}_\I)}_{\text{outlier contribution}}
\]
We apply the triangle inequality to obtain:
\[
\|\bar{\vector{x}} - \vector{\mu}_I\|
\ \leq \
\sum_{i\in \I} w_i \|\vector{x}_i - \vector{\mu}_\I\| + \sum_{i \in \Out} w_i \|\vector{x}_i - \vector{\mu}_\I\|
\ \leq \
w_{\max} \sum_{i\in \I} \|\vector{x}_i - \vector{\mu}_\I\| + \sum_{i \in \Out} w_i \|\vector{x}_i - \vector{\mu}_\I\|
\]
\end{proof}


\begin{table}[tbh]
\begin{center}
\caption{SNAP averaging in $\R^m$ (extended).}
\label{tab:vector_experiment_extended}
\textbf{$\R^{2}$}
\begin{tabular}{ll|rrrrr}
\toprule
 Outlier \% & & 10\% & 20\% & 30\% & 40\% & 49\% \\
\midrule
\midrule
\textit{Inlier Mean} &  & \textit{0.095} & \textit{0.102} & \textit{0.104} & \textit{0.110} & \textit{0.121} \\
\midrule
Mean &  & 0.720 & 1.422 & 2.128 & 2.840 & 3.476 \\
Weiszfeld &  & 0.190 & 0.379 & 0.644 & 1.052 & 1.677 \\
Component Median &  & 0.223 & 0.435 & 0.715 & 1.105 & 1.632 \\
MoM 95\% &  & 0.587 & 1.315 & 2.008 & 2.739 & 3.411 \\
MoM Sqrt &  & 0.644 & 1.360 & 2.052 & 2.782 & 3.428 \\
\midrule
\multirow[c]{2}{*}{SNAP $\ell_2$ $\kappa_l$} & $\bmad$ & 0.162 & 0.244 & 0.334 & 0.484 & 0.787 \\
 & $\bmed$ & 0.152 & 0.223 & 0.313 & 0.475 & 0.804 \\
\midrule
\multirow[c]{2}{*}{SNAP $\ell_2$ $\kappa_g$} & $\sigmamad$ & 0.139 & \textbf{0.182} & \textbf{0.228} & 0.390 & 0.756 \\
 & $\sigmamed$ & 0.142 & 0.191 & 0.244 & \textbf{0.373} & \textbf{0.711} \\
\midrule
\multirow[c]{2}{*}{SNAP $\ell_2^2$ $\kappa_l$} & $\bmad$ & 0.150 & 0.248 & 0.382 & 0.731 & 1.408 \\
 & $\bmed$ & 0.140 & 0.232 & 0.384 & 0.801 & 1.599 \\
\midrule
\multirow[c]{2}{*}{SNAP $\ell_2^2$ $\kappa_g$} & $\sigmamad$ & \textbf{0.134} & 0.196 & 0.314 & 0.815 & 1.693 \\
 & $\sigmamed$ & 0.137 & 0.207 & 0.310 & 0.680 & 1.523 \\
\bottomrule
\end{tabular}

\vspace{5mm}
\textbf{$\R^{50}$}
\begin{tabular}{ll|rrrrr}
\toprule
Outlier \% & & 10\% & 20\% & 30\% & 40\% & 49\% \\
\midrule
\midrule
\textit{Inlier Mean} &  & \textit{0.519} & \textit{0.549} & \textit{0.591} & \textit{0.635} & \textit{0.682} \\
\midrule
Mean &  & 3.588 & 7.103 & 10.645 & 14.178 & 17.358 \\
Weiszfeld &  & 0.846 & 1.637 & 2.812 & 4.741 & 8.424 \\
Component Median &  & 1.143 & 2.213 & 3.627 & 5.554 & 8.108 \\
MoM 95\% &  & 2.907 & 6.442 & 10.125 & 13.766 & 16.975 \\
MoM Sqrt &  & 3.251 & 6.734 & 10.323 & 13.905 & 17.224 \\
\midrule
\multirow[c]{2}{*}{SNAP $\ell_2$ $\kappa_l$} & $\bmad$ & 0.737 & 0.745 & 0.744 & 0.724 & 0.800 \\
 & $\bmed$ & 0.557 & 0.590 & 0.633 & 0.675 & 0.833 \\
\midrule
\multirow[c]{2}{*}{SNAP $\ell_2$ $\kappa_g$} & $\sigmamad$ & 0.637 & 0.640 & 0.641 & 0.652 & 0.911 \\
 & $\sigmamed$ & 0.540 & 0.569 & 0.609 & 0.648 & \textbf{0.699} \\
\midrule
\multirow[c]{2}{*}{SNAP $\ell_2^2$ $\kappa_l$} & $\bmad$ & 0.687 & 0.705 & 0.715 & 0.704 & 1.758 \\
 & $\bmed$ & 0.540 & 0.574 & 0.618 & 0.677 & 2.107 \\
\midrule
\multirow[c]{2}{*}{SNAP $\ell_2^2$ $\kappa_g$} & $\sigmamad$ & 0.614 & 0.621 & 0.627 & 0.642 & 3.039 \\
 & $\sigmamed$ & \textbf{0.529} & \textbf{0.559} & \textbf{0.599} & \textbf{0.639} & 0.894 \\
\midrule
\bottomrule
\end{tabular}
\end{center}
\end{table}

\subsection{Averaging in $\R^m$}
\label{app:weiszfeld}

Table~\ref{tab:vector_experiment_extended} shows the results of SNAP on vector averaging in $\R^m$ using three additional comparison methods. Component Median is the component-wise median of the vectors. MoM 95\% and MoM Sqrt are two variants of multivariate median of means \cite{Hogsgaard2025} which partitions the dataset into a number random blocks, computing a mean for each. The geometric median of these means is the final aggregation result. For MoM 95\%, the number of blocks is chosen to be optimal estimator with $95\%$ probability \cite{LugosiMendelson2019}, while MoM Sqrt selects the number of blocks by the squre root of the number of vectors, a common choice for unknown data. For comparison, the Mean was included as upper bound of the error.

One can clearly see that SNAP outperforms all variants. For most fractions of outliers, they show worse results than Weiszfeld. In the case of median of means, this is in part caused by the high number of outliers, causing most partitions to contain outliers which in turn causes high errors in the mean computation which cannot be compensated by the geometric median, leading to only a small improvement over the normal arithmetic mean.

In contrast, SNAP manages to overcome the arithmetic mean by its high-quality weights computed with the Self-Consistent Agreement Principle, allowing the non-iterative SNAP weighted mean to outperform all compared methods.


\subsection{Moving Average}
\label{app:MA}

\noindent
{\bf Method description.}
Given $\{x_t\}_{i=1}^n\!\!\subset\!\! \R$, time series smoothing seeks a smoothed time series $\{\hat{x}_t\}_{i=1}^n$ that represents the underlying trend of the original data without noise. Exponentially Moving Average (EMA) is a popular method for this task, which computes a weighted sum of previous data points, where each point is weighted by the difference in time steps by an exponentially decreasing factor $\alpha$.

To maintain the iterative nature of time series smoothing and the time-based weighting of values in EMA, we define the local disagreement scores of $\{x_{t-s}, ..., x_t\}$ in a window of length $s$ for each time step $t$ as follows:
\begin{equation}
    \DA_{i,t} \ = \ \frac{\sum_{j=t-s}^t (1-\alpha)^{t-j} |x_i - x_j|}{\sum_{i=t-s}^t \sum_{j=t-s}^t (1-\alpha)^{t-j} |x_i - x_j|}
\end{equation}
By incorporating $\alpha$ similarly to EMA, disagreements between recent values are highlighted, while disagreements with values further in the past receive less importance. Using SNAP weights:
\begin{equation}
w_{i,t} \ = \ \frac{\kappa(\DA_{i,t})}{\sum_{j=t-s}^t \kappa(\DA_{j,t})}
\end{equation}
we can then perform SNAP moving average as a simple weighted sum $\hat{x}_t = \sum_{i=t-s}^t w_{i,t} x_i$. Note that although all $x_i$ influence the final result in EMA, one can in practice restrict the total window size $s$ due to exponential decay of $\alpha$ and a typical choice is $s = 2/\alpha - 1$.

\noindent
{\bf Experimental results.}
Table~\ref{tab:MA} shows the results of SNAP on an example time series dataset, compared to standard EMA as the RMSE to the ground truth data.
The dataset consists of 100 independent time series. Each series was generated by first evaluating the ground truth function $\sin(x)$ at 200 equally spaced points over the interval $[0, 2\pi]$.
Then, Gaussian noise was added with $\sigma=0.1$ to all points, followed by an additional Gaussian noise with $\sigma=0.8$ added to 10\% of the points. An example of this dataset is shown in Figure~\ref{fig:ema}. For SNAP, the window size $s$ was restricted to $s=2/\alpha-1$. SNAP outperforms EMA for small values of $\alpha$ that are relevant for data smoothing. For completeness reason we show the experimental results for all $\alpha$ values. Figure~\ref{fig:ema} Shows an example result for the EWM dataset computed with $\alpha=0.2$.

\begin{table}[t]
\caption{{\bf Moving average.} Average RMSE results of 100 trials on an example timeseries for timeseries smoothing. Note that $\alpha > 0.5$ is seldom used in practice, but is included for completeness.}
\label{tab:MA}
\begin{center}
    \begin{small}
        \begin{sc}
\begin{tabular}{ll|rrrrrrrrr}
\toprule
$\alpha,~ k=2/\alpha-1$ & & 0.1 & 0.2 & 0.3 & 0.4 & 0.5 & 0.6 & 0.7 & 0.8 & 0.9 \\
\midrule
\midrule
EMA & & 0.1892 & 0.1300 & 0.1281 & 0.1388 & \textbf{0.1565} & \textbf{0.1751} & 0.1953 & 0.2168 & 0.2398 \\
\midrule
\multirow[c]{2}{*}{SNAP $\kappa_l(x)$} & $\bmad$ & \textbf{0.1577} & \textbf{0.1052} & \textbf{0.0997} & \textbf{0.1167} & 0.1895 & 0.2207 & 0.2207 & 0.2207 & 0.2207 \\
 & $\bmed$ & 0.1638 & 0.1080 & 0.1034 & 0.1227 & 0.1935 & 0.2207 & 0.2207 & 0.2207 & 0.2207 \\
\midrule
\multirow[c]{2}{*}{SNAP $\kappa_g(x)$} & $\sigmamad$ & 0.1706 & 0.1124 & 0.1085 & 0.1260 & 0.1616 & 0.1943 & \textbf{0.1943} & \textbf{0.1943} & \textbf{0.1943} \\
 & $\sigmamed$ & 0.1668 & 0.1104 & 0.1082 & 0.1298 & 0.1778 & 0.2064 & 0.2064 & 0.2064 & 0.2064 \\
\bottomrule
\end{tabular}
        \end{sc}
    \end{small}
\end{center}
\vskip -0.1in
\end{table}

\begin{figure}[h!]
    \centering
    \includegraphics[width=0.8\linewidth]{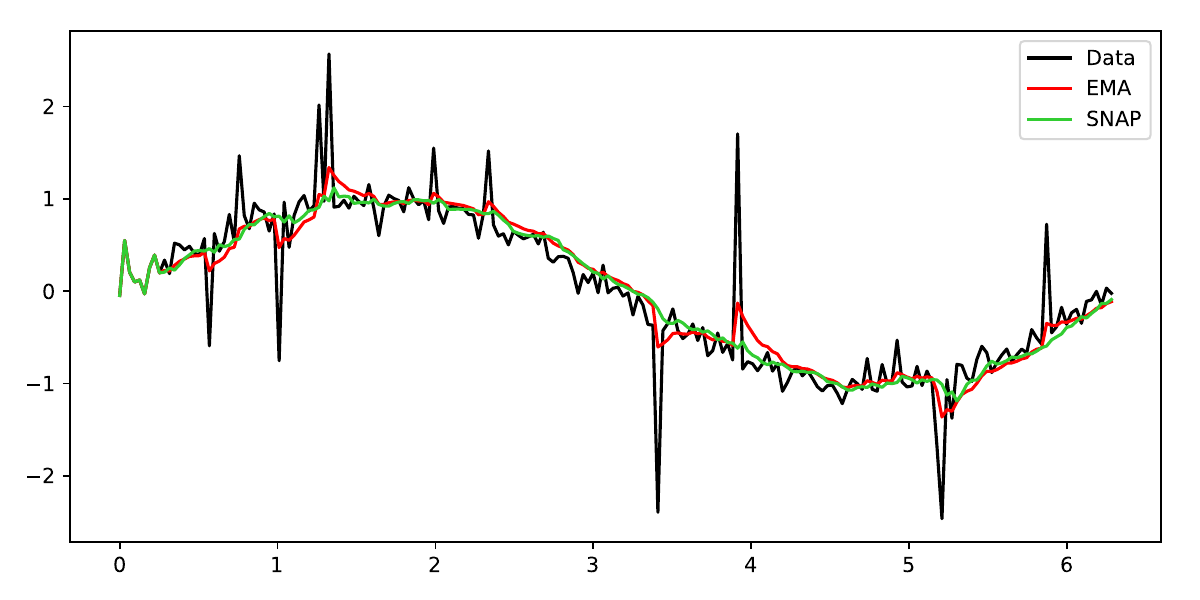}
    \caption{{\bf Moving average.} Example result for Exponential Weighted Moving Average, with and without SNAP computed weights.\label{fig:ema}}
\end{figure}

\end{document}